\documentclass[10pt]{IEEEtran}

\usepackage{amssymb}

\renewcommand{\leq}{\leqslant}

\usepackage{amsmath,mathtools,steinmetz}
\usepackage{algorithm}
\usepackage[noend]{algorithmic}
\usepackage{multirow}
\usepackage{amsthm}
\usepackage{stmaryrd}
\usepackage{paralist}
\usepackage{tikz,pgfplots,subfigure}
\usepackage{eurosym}
\usepackage{epstopdf,multicol,float}
\usepackage[hide links]{hyperref}
\usepackage{cleveref}
\usepackage{bm}
\usepackage{upgreek}
\usepackage{dsfont}
\usepackage{color}
\usepackage[colorinlistoftodos,bordercolor=orange,backgroundcolor=orange!20,linecolor=orange,textsize=scriptsize]{todonotes}

 \usetikzlibrary{external,matrix}
 \makeatletter
 \def\ps@IEEEtitlepagestyle{%
   \def\@oddfoot{\mycopyrightnotice}%
   \def\@evenfoot{}%
 }
 \def\mycopyrightnotice{%
   {\footnotesize This work has been submitted to the IEEE for possible publication. Copyright may be transferred without notice, after which this version may no longer be accessible.\hfill}
     \gdef\mycopyrightnotice{}
   }

\usetikzlibrary{calc,math,arrows, decorations.markings}

\IEEEoverridecommandlockouts           

\newcommand{\myinput}[1]{}

\newtheorem{prop}{Proposition}
\newtheorem{thm}{Theorem}
\newtheorem{cor}{Corollary}

\theoremstyle{remark}

\newtheorem{exa}{Example}
\theoremstyle{definition}
\newtheorem{definition}{Definition}

\newcommand{\dom}{\operatorname{dom}}

\newcommand{\R}{\mathbb{R}}
\newcommand{\inner}[2]{\langle #1,#2\rangle}
\newcommand{\Rp}{\R_{\geqslant 0}}
\newcommand{\Z}{\mathbb{Z}}
\newcommand{\Zp}{\mathbb{N}}

\newcommand{\bba}{\bm{\alpha}}
\newcommand{\bbdelta}{\bm{\delta}}

\newcommand{\bbz}{\mathbf{z}}
\newcommand{\bbx}{\mathbf{x}}

\newcommand{\lse}{\mathrm{LSE}}
\newcommand{\dlse}{\mathrm{DLSE}}
\newcommand{\lset}{\lse_{T}}
\newcommand{\dlset}{\dlse_{T}}
\newcommand{\Rpp}{\R_{>0}}

\newcommand{\mLD}{\mathcal{L}}

\newcommand{\bbv}{\mathbf{v}}

\newcommand{\eps}{\varepsilon}
\newcommand{\pos}{\mathrm{POS}}
\newcommand{\gpos}{\mathrm{GPOS}_T}
\newcommand{\bbq}{\mathbf{q}}

\newcommand{\bbb}{\bm{\beta}}

\newcommand{\bbd}{\bm{\delta}}
\newcommand{\bbc}{\bm{\chi}}
\newcommand{\bbg}{\bm{\gamma}}

\newcommand{\ffnn}{\mathrm{FFNN}}

\newcommand{\bbxi}{\bm{\xi}}

\newcommand{\obba}{\overrightarrow{\bba}}
\newcommand{\obbg}{\overrightarrow{\bbg}}
\newcommand{\mK}{\mathcal{K}}

\newcommand{\calD}{{\mathcal D}}

\newcommand{\beq}{\begin{equation}}
\newcommand{\eeq}{\end{equation}}

\DeclareMathOperator*{\argmin}{arg\,min}

\title{A Universal Approximation Result \\ for Difference of log-sum-exp Neural Networks}

\author{\thanks{G. C. Calafiore and C. Possieri are with the Dipartimento di Elettronica e Telecomunicazioni, Politecnico di Torino, 10129 Turin, Italy (e-mails: [giuseppe.calafiore, corrado.possieri]@polito.it).
G.C. Calafiore is also with IEIIT-CNR Torino, 10129 Turin, Italy.
S. Gaubert is with INRIA and CMAP, Ecole polytechnique, UMR 7641 CNRS, France (e-mail: stephane.gaubert@inria.fr).}
Giuseppe C. Calafiore, \IEEEmembership{Fellow,~IEEE}, Stephane Gaubert, \IEEEmembership{Member,~IEEE} \\and Corrado Possieri, \IEEEmembership{Associate Member,~IEEE}}

\begin{document}

\maketitle

\begin{abstract}
We  show that a neural network whose output is obtained
as the difference of the outputs of two feedforward networks with exponential activation function
in the hidden layer and logarithmic activation function in the output node (LSE networks) is a smooth universal approximator
of continuous functions over convex, compact sets. By using a logarithmic transform, this class of
networks maps to a family of subtraction-free ratios of generalized posynomials, which we also show to be
universal approximators of positive functions over log-convex, compact subsets of the positive orthant. 
The main advantage of Difference-LSE networks with respect to classical feedforward neural networks
is that, after a standard training phase,  they provide  surrogate models   for design that possess a specific 
 difference-of-convex-functions form, which makes them optimizable via  relatively efficient numerical methods.
In particular, by adapting an existing  difference-of-convex algorithm  to these models, we obtain an algorithm for performing effective optimization-based design. 
We illustrate the proposed approach by applying it to  data-driven design of a diet for a patient with type-2 diabetes.
\end{abstract}


\section{Introduction}
\subsection{Motivation}
A well-known and compelling  property of feedforward neural network ($\ffnn$) models is that they are capable
of approximating any continuous function over a compact set. Indeed, classical results in, e.g., \cite{cybenko1989approximation,hornik1989multilayer},
show that, given any non-constant, bounded and continuous function, there exists a $\ffnn$ 
with a single hidden layer that can approximate it over a compact set.
This \emph{universal approximation} capability, together with the development of efficient algorithms to tune the network weights,
paved the way to the efficient application of artificial neural networks in several frameworks, such as circuit design  \cite{zaabab1995neural},
control and identification of nonlinear systems \cite{221350}, optimization over graphs \cite{peterson1989new}, and many others.

However, when the goal of a neural network model is to construct a \emph{surrogate model} for describing, and then optimizing, a complex input-output relation,
it is of crucial importance that the structure of the model be well tailored for the subsequent numerical optimization phase. This is not usually the case for generic $\ffnn$. 
Indeed,  if the input-output  model does not satisfy 
certain properties (such as, for instance, convexity \cite{boyd2004convex}), then designing the input so that the output is minimized, possibly under additional design constraints, can be an extremely difficult task.

In \cite{CaGaPo:18}, we  showed that {$\lset$-networks}, that is,
$\ffnn$ with exponential activation functions in the inner layer and 
logarithmic activation function in the output neuron, parametrized
by a positive ``temperature'' parameter $T>0$,
provide a smooth convex model capable of approximating any convex function over a convex, compact set. Maps in the $\lset$ class are precisely log-Laplace
transforms of nonnegative measures with finite support, a remarkable class of maps enjoying smoothness and strict convexity properties.
We showed in particular that if the data to be approximated satisfy some
convexity assumptions, such network structures can be readily exploited to perform 
data-driven design by using convex optimization tools.  
Nevertheless, most real-life input-output maps are of nonconvex nature, hence while they might still be approximated via a convex model, such an approximation  may not yield a desirable accuracy. 

\subsection{Contribution}
The purpose of this paper is to propose a new type of neural network model, here named {\em Difference-LSE  network}
 ($\dlset$), which is constructed by taking the  difference of the outputs of two $\lset$ networks.
First, we prove that $\dlset$ networks guarantee universal approximation capabilities (thus overcoming the limitations of  plain convex  $\lset$ networks),
see \Cref{thm:dlse}.
By using a logarithmic transformation,  $\dlset$ networks map to a family of ratios of generalized posynomials functions,
which we show to be \emph{subtraction free} universal approximators of positive functions over compact subsets of the positive orthant. Subtraction free
expressions are fundamental objects in algebraic complexity, studied
in particular in~\cite{fomin}. It is a result of independent interest
that subtraction free expressions provide universal approximators.

Moreover, we show that Difference-LSE network are of practical interest,
as they have a structure which is amenable to effective optimization over the inputs by using
``DC-programming'' methods, as discussed in Section~\ref{sec:design}. 

Training and subsequent optimization of $\dlset$ networks has been implemented 
in a numerical Matlab toolbox\footnote{See \url{https://github.com/Corrado-possieri/DLSE_neural_networks}.} named \texttt{DLSE\_Neural\_Network} that we made publicly  available.
The theoretical results in the paper are illustrated 
 by an example 
dealing with  data-driven design of a diet for a patient with type 2 diabetes.

\subsection{Related work}
Similar to our previous work~\cite{CaGaPo:18} on which it builds,
the present paper is inspired by ideas from tropical geometry and max-plus algebra.
The class of functions in $\lse$ that we study here plays a key role in Viro's patchworking methods~\cite{viro,itenberg} for real curves. 
We note that the application of tropical geometry to neural networks
is an emerging topic: 
at least two recent works have used tropical methods
to provide combinatorial estimates, in terms of Newton polytopes,
of the ``classifying power'' of neural networks
with piecewise affine functions, see~\cite{charisopoulos2017morphological}, \cite{lim}. Other related results concern the ``zero-temperature'' ($T=0$)
limit of the approximation problem that we consider, i.e., the representation of piecewise linear functions by elementary expression involving min, max, and affine terms~\cite{ovchinnikov,wang}, and the approximation of functions
by piecewise linear functions. In particular, Th.~4.3 of \cite{Goodfellow}
shows that any continuous function can be approximately
arbitrarily well on a compact domain by a difference
of piecewise linear convex functions.
A related approximation result, still concerning differences
of structured piecewise linear convex functions, has appeared in~\cite{1903.08072}. In contrast, our results provide approximation
by differences of a family of {\em smooth} structured convex functions.
This smoothing is essential when deriving
universal approximation results by subtraction
free rational expresssions.

\section{Notation and technical preliminaries\label{sec:notation}}
Let $\Zp$, $\Z$, $\R$, $\Rp$, and $\Rpp$ denote the set of natural, integer, real, nonnegative real,
and positive real numbers, respectively.
%
Given a function $\phi:\R^n\rightarrow\R\cup\{+\infty\}$, we define its domain as 
 $\dom \phi\doteq\{\bbx\in\R^n:\phi(\bbx)<+\infty \}$.
If the function $\phi:\R^n\rightarrow\R\cup\{+\infty\}$ is differentiable at
point $\bbx$, we denote by $\nabla\phi(\bbx)$ its  {gradient} at $\bbx$.


\subsection{Log-Sum-Exp functions}
\label{subsec:lse}
Following~\cite{CaGaPo:18},
we define $\lse$ (Log-Sum-Exp) as the  class of  functions $f:\R^n\rightarrow\R$ that can be written as
\begin{equation}
f(\bbx) = \log \left(\sum_{k=1}^K b_k \exp ( \inner{\bba^{(k)}}{\bbx} )\right),
\label{eq:lse_lse_reg}
\end{equation}
for some $K\in\Zp$, $b_k\in\Rpp$, $\bba^{(k)}=[\begin{array}{ccc}
\alpha_1^{(k)} & \cdots & \alpha_n^{(k)}
\end{array}]^\top\in\R^n$, $k=1,\dots,K$, where
 $\bbx=[\begin{array}{ccc}
x_1 & \cdots & x_n
\end{array}]^\top$ is a vector of variables. Further, given $T\in\Rpp$, we define  $\lset$ as the class of  functions $f_T:\R^n\rightarrow\R$ that can be written as
\begin{equation}
f_T(\bbx) = T \log \left(\sum_{k=1}^K b_k^{1/T} \exp ( \inner{\bba^{(k)}}{\bbx/T} )\right),
\label{eq:lse_lse_reg_trop}
\end{equation}
for some $K\in\Zp$, $b_k\in\Rpp$, and $\bba^{(k)}\in\R^n$, $k=1,\dots,K$. 
By letting
$\beta_k \doteq \log b_k$, $k=1,\ldots,K$,
we have that functions in the family  $\lset$ can be equivalently parameterized as
\begin{equation}
f_T(\bbx) = T \log \left(\sum_{k=1}^K \exp ( \inner{\bba^{(k)}}{\bbx/T} + \beta_k/T )\right),
\label{eq:lse_lse_reg_trop_exp}
\end{equation}
where the $\beta_k$s have no sign restrictions.
It may sometimes be convenient to highlight the full parameterization 
of $f_T$, in which case we shall write $f_T^{(\obba,\bbb)}$,
where $\obba = (\bba^{(1)},\ldots,\bba^{(K)})$, and
$\bbb =  (\beta_1,\ldots,\beta_K)$.
It can be readily observed that,  for any $T>0$, the following property holds:
\beq
f_T^{(\obba,\bbb)} (\bbx) =  T f_1^{(\obba,\bbb/T)} (\bbx/T).
\label{eq:scaling}
\eeq
The maps in $\lset$ are special instances of the log-Laplace transforms
of nonnegative measures, studied in~\cite{klartag2012centroid}.  In particular,
the maps in $\lset$ are smooth, they are convex (this is an easy consequence of Cauchy-Schwarz inequality), and they are even strictly convex if the vectors $\obba^{(k)}$ constitute an affine generating family of $\R^n$, see~\cite[Prop.~1]{CaGaPo:18}.
Maps of this kind play a key role in tropical geometry,
in the setting of Viro's patchworking method~\cite{viro},
dealing with the degeneration of real algebraic curves to a piecewise
linear limit. We note in this respect that the family of functions $(f_T)_{T>0}$ given by~\eqref{eq:lse_lse_reg_trop_exp} converges uniformly on $\R^n$, as $T\to 0^+$, to the function
\[
f_0(\bbx) \doteq  \max_{1\leq k\leq K} \big(
\inner{\bba^{(k)}}{\bbx} + \beta_k\big).
\]
Actually, the following inequality holds for all $T>0$
\begin{equation}
  f_0(\bbx) \leq f_T(\bbx) \leq T\log K + f_0(\bbx),
  \label{eq:lse_lse_reg_trop_exp-approx}
\end{equation}
see \cite{CaGaPo:18} for details and background.

\subsection{Posynomials and $\gpos$ functions}
Given $c_k > 0$ and $\bba^{(k)}\in\R^n$, a \emph{positive monomial}
is a product of the form $c_k\bbx^{\bba^{(k)}} = c_k x_1^{\alpha_1^{(k)}}x_2^{\alpha_2^{(k)}}\cdots x_n^{\alpha_n^{(k)}}$.
A \emph{posynomial} is a finite sum of positive monomials,
\begin{equation}
\psi(\bbx) = \sum_{k=1}^K  c_k \bbx^{\bba^{(k)}}.
\label{eq:POS}
\end{equation}
We denote by $\pos$ the class of  functions $\psi:\Rpp^n\rightarrow\Rpp$ of the form
(\ref{eq:POS}).
 Posynomials are {\em log-log-convex} functions, meaning that the log of a posynomial $\psi$ is convex in the log of its argument, see, e.g., Section~II.B of \cite{CaGaPo:18}.
  We denote by $\gpos$ the class of functions that can be expressed as
  \begin{equation}
\psi_T(\bbx) = (\psi(\bbx^{1/T}))^T
\label{eq:genposy}
\end{equation}
for some  $T>0$ and $\psi\in\pos$.
 These functions are log-log-convex, and they form a subset of the so-called generalized posynomial functions, see, e.g., Section~II.B of \cite{CaGaPo:18}.
It is observed in Proposition 3 of \cite{CaGaPo:18} that  $\lset$ and $\gpos$ functions are related by a one-to-one correspondence. That is,
for any $f(\bbx)\in\lset$ and $\psi(\bbz)\in \gpos$ it holds that
\begin{align*}
\exp \left( f\left( \log(\bbz)\right)  \right) & \in \gpos, &
\log \left( \psi \left( \exp(\bbx)\right)  \right)& \in \lset.
\end{align*}

\section{A universal approximation theorem}
\label{sec_DLSE}

\subsection{Preliminary: approximation of convex functions}
\label{ssec_convexapprox}
We start by recalling a key result of \cite{CaGaPo:18}, stating  that functions in $\lset$ are universal smooth approximators of convex functions,  see Theorem~2 in \cite{CaGaPo:18}.

\begin{thm}[Universal approximators of convex functions,  \cite{CaGaPo:18}]\label{thm:LSEapprox}
  Let $\phi$ be a real valued continuous convex function defined
  on a compact convex subset ${\mathcal K} \subset \R^n$.
Then, for all $\eps > 0$ there exist $T>0$ and a function $f_T \in \lse_T$
such that
\begin{align}
  |f_T(\bbx) - \phi(\bbx)| \leqslant \eps,\quad\text{ for all }\; \bbx\in {\mK} .\label{e-defap}
\end{align}
\end{thm}
We now extend the above result by showing that it actually holds for the restricted class of $ \lse_T$ with rational parameters. This extension will allow
us to apply our results to the approximation by subtraction
free expressions. 

\begin{definition}
A function $f_T\in\lse_T$ has {\em rational parameters} if  $T>0$ is a rational number and $f_T=f_T^{(\obba,\bbb)}$
is of the form~(\ref{eq:lse_lse_reg_trop_exp})
 where the vectors $\bba^{(1)},\ldots,\bba^{(K)}$ have rational entries, and $\beta_1,\dots,\beta_K$ are rational numbers.
We shall also say that $f_T$ is a $\varepsilon$-approximation of $f$ on $\mathcal{K}$ when~\eqref{e-defap} holds.
\end{definition}
The following corollary holds.

\begin{cor}[LSE approximation with rational parameters]
\label{cor:ratLSE}
Under the hypotheses of \Cref{thm:LSEapprox}, 
for all $\eps > 0$ there exists a rational 
$T>0$ and
a function $f_T \in \lse_T$ with rational parameters  such that 
\eqref{e-defap} holds. Moreover, $T$ may be chosen of the form $1/p$
where $p$ is a positive integer.
 \end{cor}

\begin{proof}
First, inspecting the proof of Theorem~2 in \cite{CaGaPo:18}, one obtains that
$\phi$  can be approximated uniformly  by a map $f_T\in \lse_T$,
for all $T>0$ small enough, hence we can always assume that $T$ is
of the form $1/p$ for some positive integer $p$.
It then remains to be proved that the approximation result still holds if also 
$\bba^{(1)},\ldots,\bba^{(K)}$ and $\beta_1,\dots,\beta_K$ are rational.
To this end, let us study the effect of a perturbation of these parameters
 on the map $f_T$.
 Observe that the map $\varphi: \R^K\to \R$, 
 $\bbxi=(\xi_1,\dots,\xi_K)\mapsto T \log (\sum_{k=1}^{ K} \exp(\xi_k/T))$ satisfies
  \begin{align}
    |\varphi(\bbxi)-\varphi(\bbxi')|\leqslant \|\bbxi-\bbxi'\|_\infty\label{e-sup}
  \end{align}
  for all $\bbxi\in \R^K$, where $\|\cdot\|_\infty$ is the sup-norm.
 This follows from the fact that $\varphi$
  is order preserving and commutes with the addition of a constant, see
 e.g.\ \cite{CT80} and Section~2 of~\cite{1605.04518}. It follows from~\eqref{e-sup} that if
  $f_T$ is as in~\eqref{eq:lse_lse_reg_trop_exp}, and if
  \[
  g_T(\bbx) = T \log \left(\sum_{k=1}^K \exp \Big( \inner{\bbg^{(k)}}{\bbx/T} + \delta_k/T \Big)\right) ,
  \]
  then, letting
  $R \doteq \max_{\bbx\in \mathcal{K}} \|\bbx\| $,
 $ \obbg=(\bbg^{(1)},\ldots, \bbg^{(K)})$, and $\bbdelta = (\delta_1,\ldots,\delta_K)$,
  we get
  \begin{align*} 
    |f_T(\bbx) - g_T(\bbx)|&\leqslant \kappa((\obba,\bbb),(\obbg,\bbdelta))\\
    & \!\!\!\!\!\!\!\!\!\doteq  \max_{1\leqslant k\leqslant K} \|\bba^{(k)}-\bbg^{(k)}\|R
  + \max_{1\leqslant k\leqslant K} |\beta_k-\delta_k|  ,
  \end{align*}
  for all $\bbx\in \mathcal{K}$.
  Hence, choosing $(\obbg,\bbdelta)$ to be a rational approximation
  of $(\obba,\bbb)$ such that $\kappa((\obba,\bbb),(\obbg,\bbdelta))\leqslant \varepsilon$, and supposing that $f_T$ is a $\varepsilon$-approximation of $\phi$ on $\mathcal{K}$, we deduce that $g_T$ is a $2\varepsilon$-approximation of $\phi$ on $\mathcal{K}$, from
  which the statement of the corollary follows.
  \end{proof}

\subsection{Approximation of general continuous functions}
\label{ssec_generalapprox}
This section contains our main result on universal approximation of continuous functions.
To this end, we first define the class of functions that can be expressed as the difference of two
functions in $\lse_T$.

\begin{definition}[$\dlse_T$ functions]
We say that a function $\phi:\R^{n}\to \R$ belongs to the $\dlse_T$ class, if
$\phi = g_T - h_T$, for some $g_T, h_T\in\lse_T$.
Further, we say that  $\phi$ has rational parameters, if 
 $g_T$ and $h_T$ have rational parameters.
\end{definition}

The following result  shows
that any  continuous function can  be
approximated uniformly by a function in a $\dlse_T$ class.

\begin{thm}[Universal approximation property of $\dlse_T$]\label{thm:dlse}
  Let $\phi$ be a real-valued continuous function defined
  on a compact, convex subset ${\mathcal K} \subset \R^n$.
Then, for any $\eps > 0$ there exist
a function $f_T\in \dlse_T$
with rational parameters, for some $T=1/p$ where $p$ is a positive integer,
such that $|f_T(\bbx) -  \phi(\bbx)| \leqslant \eps$, $\forall\bbx\in {\mathcal K}$. 
\end{thm}

\begin{proof}
  A classical result of convex analysis states that any continuous function $\phi$ defined
  on a compact convex subset $\mathcal{K}$ of $\R^n$ can be written as the difference $g-h$ where $g,h$ are continuous, convex functions defined on $\mathcal{K}$, see, e.g., Proposition~2.2 of \cite{borwein}. Then, by \Cref{cor:ratLSE},
  for all $\varepsilon>0$, we can find a rational $T'>0$ and a function $g_{T'}\in\lse_{T'}$ with rational parameters
  such that $|g(\bbx)-g_{T'}(\bbx)|\leqslant \varepsilon/2$ holds
  for all $\bbx\in\mathcal{K}$. 
  Similarly, we can find
  a rational $T''  >0$ and a function $h_{T''}\in\lse_{T''}$ with rational parameters
  such that $|h(\bbx)-h_{T''}(\bbx)|\leqslant \varepsilon/2$ holds
  for all $\bbx\in\mathcal{K}$. 
  Hence, by taking any rational $T>0$ such that $T'$ and $T''$ are integer
  multiples of $T$, it follows from the nesting property in Lemma~1 of \cite{CaGaPo:18}
  that  $g_{T'}$ and $h_{T''}$ both belong to
  $\lse_{T}$. Thus, there exist a rational $T>0$ and
  $g_T,h_T\in\lse_T$ such that, for all $\bbx\in\mathcal{K}$,
 \begin{align*} 
    \vert  g(\bbx)-g_{T}(\bbx) \vert  &\leqslant \varepsilon/2 &
    \vert  h_{T}(\bbx)  -h(\bbx)\vert &\leqslant  \varepsilon/2.
 \end{align*} 
  Summing these conditions we obtain that 
  $| (g(\bbx)- h(\bbx)) - (g_{T}(\bbx) -h_{T}(\bbx)  ) |\leqslant \varepsilon$, for all $\bbx\in\mathcal{K}$.
   The claim then immediately follows 
    by recalling that $g(\bbx)- h(\bbx) = \phi(\bbx)$ for all $\bbx\in\mathcal{K}$, and letting $f_T\doteq g_{T} -h_{T}$, whence $f_T\in \dlse_T$.
\end{proof}
The following explicit example illustrates the approximation
of a non-convex and nondifferentiable function
by a function in $\lset$.
\begin{exa}
  Consider
  \[ \phi(x)=\max(0,\min(x,1))
   .
  \]
  Observe that $\phi(x)=\max(0,x)-\max(0,x-1)$, which is indeed  a difference of two nonsmooth convex
  functions.   By using~\eqref{eq:lse_lse_reg_trop_exp-approx}, we can approximate
  each term of this difference by a function in $\lse$ as
  \begin{align*}
  \max(0,x) &\leq T\log(1+\exp( x/ T)) \\
  &\qquad \leq T\log 2+ \max(0,x),\\
  \max(0,x-1) &\leq T\log(1+\exp({(x-1)}/T)) \\
  &\qquad \leq T\log 2+ \max(0,x-1).
  \end{align*}
  
  It follows that the map
  \[ f_{T}\doteq T(\log(1+\exp( x/ T)\big)- \log\big(1+\exp({(x-1)}/T)) )
  \]
  is in $\dlse_{T}$ and
  satisfies the following uniform approximation property
  of $\phi$:
  \[ -T\log 2 + f_{T}(x) \leq \phi(x) \leq T\log 2 + f_{T}(x),\qquad \forall x\in \R^n .
  \]
\end{exa}
\begin{exa}
  The previous explicit approximation carries over to a continuous piecewise affine
  function $\phi$ of a single real variable, as follows.
  By {\em piecewise affine}, we
  mean that $\R$ can be covered by finitely many intervals
  in such a way that $\phi$ is affine over each of these intervals.
  Then, $\phi$
  can be written in a unique way as
  \begin{align}
    \phi(x) = ax+b + \sum_{1\leq i\leq K} \alpha_i \max(0,x-\gamma_i)
    \label{e-dec-phi}
  \end{align}
  where $a,b$ are real parameters, $\gamma_1<\dots<\gamma_K$ are the nondifferentiability
  points of $\phi$, and $\alpha_i= \phi'(\gamma_i^+) -\phi'(\gamma_i^-)$
  is the jump of the derivative of $\phi$ at point $\gamma_i$.
  Another way to get insight of~\eqref{e-dec-phi}
    is to make the following  observation:
the function $\phi$
  has a second derivative in the distribution sense, $\phi''= \sum_{1\leq i\leq K} \alpha_i \delta_{\gamma_i}$; then~\eqref{e-dec-phi} is gotten by integrating
  twice the latter expression of $\phi''$. Possibly after subtracting
  to $\phi$ an affine function, we will always assume that $a=b=0$.
  
Then, setting 
\begin{align*}
I^+&\doteq  \{1\leq i\leq k\mid \alpha_i>0\},\\
I^-&\doteq  \{1\leq i\leq k\mid \alpha_i<0\},
\end{align*}
and   \[
  \phi^\pm(x) =     \sum_{i\in I^\pm} |\alpha_i| \max(0, x-\gamma_i)
  \]
  we write
\begin{equation*}
   \phi(x) = \phi^+(x)-\phi^-(x) .
\end{equation*}
Let 
\begin{align*}
a_i &= \sum_{j\in I^\pm, j\leq i} |\alpha_j|,  \\
b_i &= \sum_{j\in I^\pm, j\leq i} |\alpha_j|\gamma_j,
\end{align*}
and note that
\[
\phi^\pm (x) =\max_{i\in I^\pm} (a_i x -b_i) .
\]
Then, setting $f_T\doteq  f_T^+-f_T^-$
where
\[
f_T^\pm(x) = T\log\left(\sum_{i\in I^\pm} \exp({(a_ix-b_i)}/{T})\right)
\]
and using~\eqref{eq:lse_lse_reg_trop_exp-approx},
we get
\[
f_T - T\log |I^-|
\leq f \leq f_T + T\log |I^+|  .
\]
  \end{exa}
  
  \subsection{Data approximation}
  Consider a collection $\calD$ of $m$ data-points,
  \beq
  \calD= \{(\bbx_i,y_i)\}_{i=1}^m
  \label{eq:Ddataset}
\eeq
  where $y_i = \phi(\bbx_i)$, $i=1,\ldots,m$, and $\phi$ is an {unknown} function.
  The following universal data approximation result holds.

\begin{cor}[Universal data approximation]\label{cor:convAppr}
Given a collection of  data $\calD$ as in \eqref{eq:Ddataset},
for any $\eps>0 $ there exists ${T}>0$
and a function $d_T\in \dlset$ with rational coefficients such that 
\[
|d_T(\bbx_i)-y_i|\leqslant \eps,\quad i=1,\dots,m.
\]
\end{cor}

\begin{proof}
Let $\mathcal K \doteq  \mbox{co}\{\bbx_i,\ldots, \bbx_m\}$ be the convex hull of the input data points.
Consider a triangulation of the input points $\bbx_i$: recall that
such a triangulation consists of a finite collection of simplices $(\Delta_r)_{r\in R}$, satisfying the following properties:
(i) the vertices of these simplices are taken among the points $\bbx_1,\dots,\bbx_m$; (ii) each point $\bbx_i$ is the vertex of at least one simplex;
(iii) the interiors of theses simplices have pairwise empty intersections,
and (iv) the union of these simplices is precisely $\mathcal K$. Then, there
is a unique continuous function, $f$, affine on each simplex $\Delta_r$,
and such that $f(\bbx_i) = y_i$ for $1\leq i\leq m$. 
Observe that $\mathcal K$ is convex and compact by construction.
Now, a direct application of  \Cref{thm:dlse} 
shows that for any $\eps>0 $ there exists ${T}>0$
and a function $d_T\in \dlset$ with rational coefficients such that 
\[
|d_T(\bbx_i)- f(\bbx_i) | = |d_T(\bbx_i)-y_i|\leqslant \eps,\quad i=1,\dots,m,
\]
which concludes the proof. 
\end{proof}

\section{Positive functions on the positive orthant}
\label{sec:posorth}

In this section we discuss approximation results for functions taking positive values on the open positive orthant.
A particular case of this class of functions is given by log-log-convex functions, whose uniform approximation by means of
$\gpos$ functions was discussed in Corollary~1 of \cite{CaGaPo:18}.
We  shall first extend this result to functions with rational parameters, and then provide a universal approximation result for continuous  positive functions over the open positive orthant.
\subsection{Uniform approximation results}
The following preliminary definitions are instrumental for our purposes: a  
 subset ${\mathcal R}\subset \Rpp^n$ will be said to be {\em log-convex}
if its image by the map that takes the logarithm entry-wise is convex.
We shall say that a function $\psi_T\in\gpos$ has { rational parameters}
if it can be written as in~\eqref{eq:genposy} with $\psi$ given by~\eqref{eq:POS},
in such a way that $T$, the entries of the vectors $\bba^{(1)}$,
\dots, $\bba^{(K)}$, and the scalars $\log c_1,\dots,\log c_K$ are rational
numbers.
The following corollary extends Corollary~1 of~\cite{CaGaPo:18}.

\begin{cor}
[Universal approximators of log-log-convex functions]\label{cor:prelimloglogcvx}
Let $\ell$ be a log-log-convex function defined on a compact,
log-convex subset
${\mathcal R}$ of $ \Rpp^n$. 
Then, for all $\tilde{\eps} > 0$ there exist a function $\psi_T \in \gpos$ with rational parameters, for some $T=1/p$ where $p$ is a positive integer,
such that, for all $\bbx\in\mathcal{R}$,
\begin{equation}\label{eq:relBoun}
\left\vert \frac{\ell(\bbx)-\psi_T(\bbx)}{\min(\ell(\bbx),\psi_T(\bbx))}\right\vert\leqslant \tilde{\eps}.
\end{equation}
\end{cor}
\begin{proof}
By using the log-log transformation, define $\tilde{\ell}(\bbq)\doteq\log(\ell(\exp(\bbq)))$. Since $\ell(\bbx)$ is
log-log-convex in $\bbx$, $\tilde{\ell}(\bbq)$ is convex in $\bbq=\log \bbx$. 
Furthermore, the set $\mathcal{K}\doteq \log(\mathcal{R})$
is convex and compact since the set $\mathcal{R}$ is log-convex and compact.
Thus, by \Cref{cor:ratLSE}, for all ${\eps}>0$, there exist  $T=1/p$ where $p$ is a positive integer, and a function $f_T \in \lse_T$ with rational coefficients such that $|f_T(\bbq)-\tilde{\ell}(\bbq)|\leqslant {\eps}$ for all $\bbq\in\mathcal{K}$. 
From this point on, the proof follows the very same lines as the proof of Corollary~1 of~\cite{CaGaPo:18}.
\end{proof}
We next state an approximation result for functions on the positive orthant. The derivation of \Cref{thm:mainresult3} from \Cref{cor:prelimloglogcvx} is similar to the derivation of 
\Cref{thm:dlse} from \Cref{cor:ratLSE} and thus we omit its proof.
\begin{thm}
[Universal approximators of functions on the open orthant]\label{thm:mainresult3}
Let $\ell$ be a continous positive function defined on a compact
log-convex subset
${\mathcal R}\subset \Rpp^n$.
Then, for all $\tilde{\eps} > 0$ there exist two functions $\psi_T,\psi'_T \in \gpos$ with rational parameters, for some $T=1/p$ where $p$ is a positive integer,
such that, for all $\bbx\in\mathcal{R}$,
\begin{equation} 
\left\vert \frac{\ell(\bbx)-\psi_T(\bbx)/\psi'_T(\bbx)}{\min(\ell(\bbx),\psi_T(\bbx)/\psi'_T(\bbx))}\right\vert\leqslant \tilde{\eps}.
\end{equation}
\end{thm}
\subsection{Universal approximation by subtraction-free expressions}
We next derive from \Cref{thm:mainresult3} an approximation result
by {\em subtraction-free expressions}. The latter
are an important subclass of rational expressions, studied in~\cite{fomin}.
Subtraction-free expressions are well formed expressions in several
commutative variables $x_1,\dots,x_n$, defined
using the operations $+,\times, /$
  and using positive constants, but not using subtraction. Formally,
  a subtraction-free expression in the variables $x_1,\dots,x_n$ is a term
  produced by the context-free grammar rule
  \[ E\to E+E, E\times E, E/E, C, x_1,\dots,x_n
  \]
  where $C$ can take
  the value of any positive constant.
  For instance, $E_1\doteq (x_1+x_2^3)/(2x_1+3x_2/(x_1+x_2))$
  is a subtraction-free expression, whereas $E_2\doteq x_1^2-x_1x_2+x_2^2$
  is not a subtraction-free expression, owing to the
  presence of the $-$ sign. Note that $E_2$, thought of as a formal rational fraction, coincides with $E_3\doteq  (x_1^3+x_2^3)/(x_1+x_2)$ which is subtraction
  free, i.e., an expression which is not subtraction-free
  may well have an equivalent subtraction-free expression.
  However, there are rational fractions, and even polynomials,
  like $E_4\doteq (x_1-x_2)^2$,
  without subtraction equivalent free expressions, because
  any subtraction-free expression must take
  positive values on the interior of the positive cone,
  whereas $E_4$ vanishes on the line $x_1=x_2$.
  Important examples of subtraction-free expressions arise from
  series-parallel composition rules for resistances. More advanced
  examples, coming from algebraic combinatorics, are discussed in~\cite{fomin}.

  \begin{cor}[Approximation by subtraction-free expressions]
Let $\ell$ be a continous positive function defined on a compact
log-convex subset
${\mathcal R}\subset \Rpp^n$.
Then, for all $\tilde{\eps} > 0$ there exist positive integers $p,q$ and
a subtraction-free expression $E$ in $n$ variables $y_1,\dots,y_n$ such that the function
\[
f(\bbx) = E(x_1^{1/q},\dots,x_n^{1/q})^{1/p}
\]
in which $x_i^{1/q}$ is substituted to the variable $y_i$, satisfies,
for all $\bbx\in\mathcal{R}$,
\begin{equation}\label{eq:relBounnew}
\left\vert \frac{\ell(\bbx)-f(\bbx)}{\min(\ell(\bbx),f(\bbx))}\right\vert\leqslant \tilde{\eps}.
\end{equation}
  \end{cor}
  \begin{proof}
    \Cref{thm:mainresult3} shows that~\eqref{eq:relBounnew}
    holds with $f=\psi_T/\psi_T'$ where $T=1/p$ for some 
    positive integer $p$, and $\psi_T,\psi'_T$ are functions
    in $\gpos$ with rational parameters,
    i.e.,
    \begin{align*}
    \psi_T(\bbx) &= \left(\sum_{k=1}^K  c_k \bbx^{\bba^{(k)}p}\right)^{1/p},
    \\
    \psi'_T(\bbx) &= \left(\sum_{k=1}^{K'}  c'_k \bbx^{(\bba')^{(k)}p}\right)^{1/p},
    \end{align*}
    where the vectors $\bba^{(k)}$ and $(\bba')^{(k)}$ have rational
    entries. Denoting by $q$ the least common multiple of the denominators
    of the entries of the vectors $\bba^{(k)}p$ and $(\bba')^{(k)}p$,
    we see that $\psi_T(\bbx)/\psi_T(\bbx)$ is precisely of the form $E(x_1^{1/q},\dots,x_n^{1/q})^{1/p}$ where $E$ is a subtraction-free rational expression.
  \end{proof}


\subsection{Approximation of  positive data}
\label{ssec:approxpos}
Consider a collection $\mLD$ of $m$ data pairs,
\begin{equation*}
\mLD = \{(\bbz_i,w_i)\}_{i=1}^m,
\end{equation*}
where $\bbz_i\in\Rpp^n$, $w_i\in\Rpp$, $i=1,\dots,m$, with
$w_i = \ell(\bbz_i)$, $i=1,\dots,m$,
where $\ell:\Rpp^n\rightarrow \Rpp$ is an unknown function. 
The data in $\mLD$ is referred to as \emph{positive data}.
{The following proposition is now an immediate consequence of
\Cref{thm:mainresult3}, where $\mathcal R$ can be taken as the log-convex hull of the
input data points\footnote{For given $\bbz_1,\ldots,\bbz_m\in\Rpp^n$, we define their log-convex hull as the set of vectors $\bbz= \prod_{i=1}^m\bbz_i^{\xi_i}$, where
$\xi_i\in[0,1]$, $i=1,\dots,m$, and $\sum_{i=1}^m\xi_i = 1$.}.

\begin{prop}\label{cor:posappr}
Given positive data $\mLD\doteq\{(\bbz_i,w_i)\}_{i=1}^m$, for any $\tilde{\eps}>0$ there exist 
 a rational $T>0$ and two functions $\psi_T,\psi'_T \in \gpos$ with rational parameters
such that
\begin{equation}\label{eq:relposdatafit}
\left\vert \frac{w_i-\psi_T(\bbz_i)/\psi'_T(\bbz_i)}{\min(w_i,\psi_T(\bbz_i)/\psi'_T\bbz_i))}\right\vert\leqslant \tilde{\eps},\quad i=1,\ldots,m.
\end{equation}
\end{prop}

\section{$\dlset$  networks \label{sec:networks}}
Functions in $\dlset$ can
be modeled through a feedforward neural network (\emph{$\ffnn$}) architecture, composed of two $\lset$ networks in parallel, whose outputs are fused via an output difference node, see
Fig.~\ref{fig:dlsetnet}. 

\begin{figure}[htb]
\centering
\resizebox{0.4\textwidth}{!}{
\includegraphics{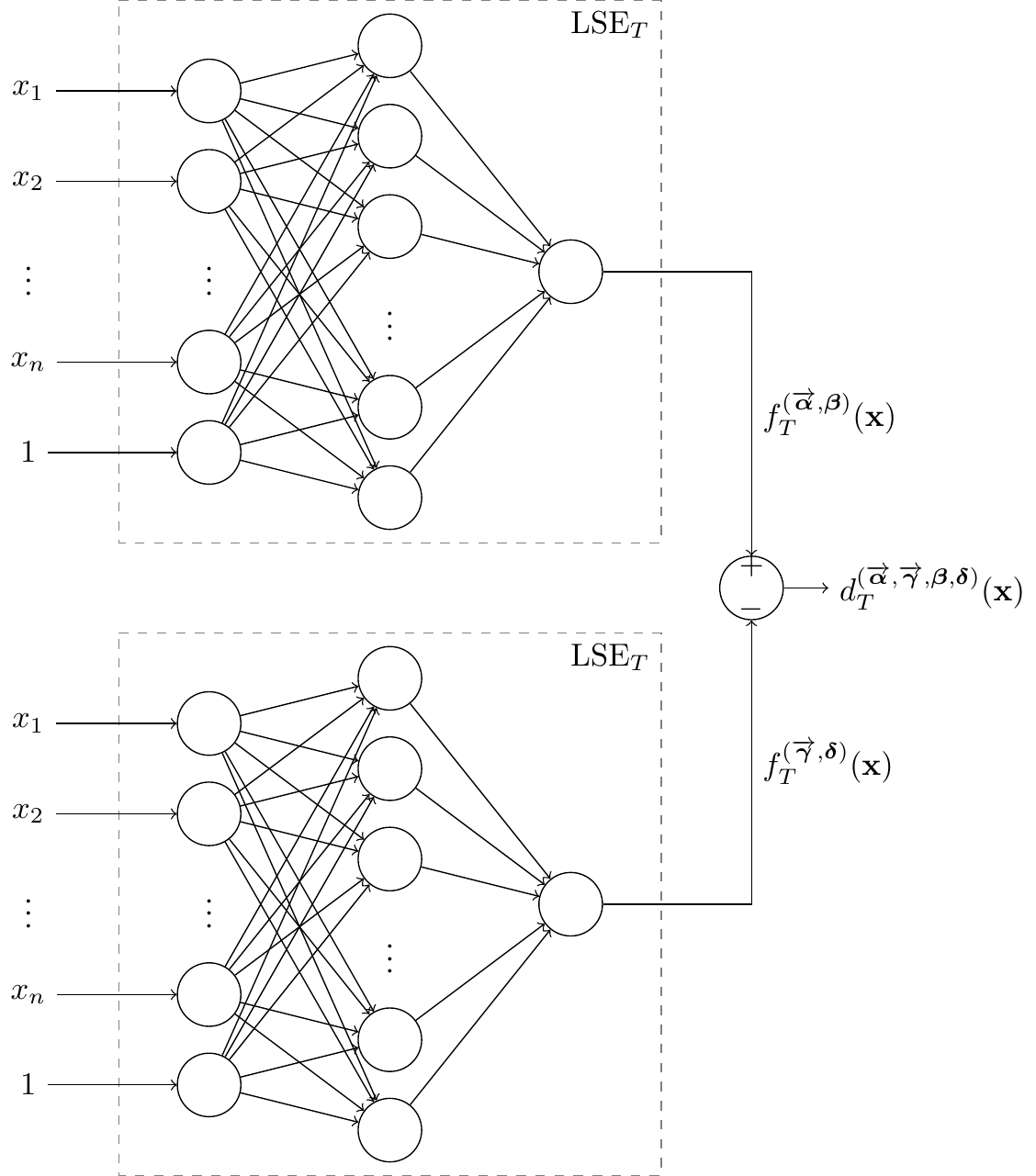}}
\caption{A $\dlset$ network is composed of two $\lset$ networks in parallel, with a difference output node. \label{fig:dlsetnet}}
\end{figure}

It may sometimes be convenient to highlight the full parameterization 
of the input-output function $d_T(\cdot)$ synthesized by the $\dlset$ network, in which case we shall write 
\[d_T^{(\obba,\obbg,\bbb,\bbd)}(\bbx)=f_T^{(\obba,\bbb)}(\bbx)-f_T^{(\obbg,\bbd)}(\bbx)\]
where $\obba = (\bba^{(1)},\ldots,\bba^{(K)})$, $\obbg = (\bbg^{(1)},\ldots,\bbg^{(K)})$, 
$\bbb =  (\beta_1,\ldots,\beta_K)$, and $\bbd=(\delta_1,\dots,\delta_k)$ are
the parameter vectors of the two $\lset$ components.

Each  $\lset$ component  has  $n$ input nodes, one hidden layer with $K$ nodes, and one output node.
The activation  function of the hidden nodes is
$ s\mapsto (\exp(s/T))$,
 and  the activation of the output node of each $\lset$ component is
 $s\mapsto T\log(s)$.
Each node in the hidden layer  of the first $\lset$ component network computes a term of the form  
$s_k =  \inner{\bba^{(k)}}{\bbx} +  \beta_k$,
where the $i$-th entry $\alpha^{(k)}_i$ of $\bba^{(k)}$ represents the weight between node $k$ and input $x_i$, and $\beta_k$ is the bias term of node $k$. Each node $k$ of the first $\lset$ network thus generates activations
\[a_k = \exp (  \inner{\bba^{(k)}}{\bbx/T} +  \beta_k/T ).\]  
We consider the weights from the inner nodes to the output node to be unitary, whence the output node of the first $\lset$ network computes
$s = \sum_{k=1}^K a_k$ and then, according to the output activation function, the output layer returns
the value  
\[
 T \log (s) = T \log \left(\sum_{k=1}^K a_k\right).
\]
An identical reasoning applies to the output of the second component  $\lset$ network. The overall output 
realizes a $\dlset$ function which, by \Cref{thm:dlse}, allows us to approximate any 
continuous function over a compact convex domain.  Similarly, by \Cref{cor:convAppr}, we
can approximate data $\calD= \{(\bbx_i,y_i)\}_{i=1}^m$ via a  $\dlset$ network, to any given precision.

\begin{thm}\label{cor:ApprFFNN}
Given a collection of data $\calD\doteq\{(\bbx_i,y_i)\}_{i=1}^m$,
for each $\eps>0$ there exists a $\dlset$ neural network 
such that \[|d_T(\bbx_i)-y_i|\leqslant \eps,\quad i=1,\dots,m,\]
where $d_T$ is the input-output function of the network.
\end{thm}


\subsection{Training $\dlset$ networks}
By using the scaling property \eqref{eq:scaling} of $\lset$ functions, it can be noticed that a simpler $\lse$ network structure can 
be used to implement a $\dlset$ neural network, as shown in Fig.~\ref{fig:dlsenet}. 

\begin{figure}[htb]
\centering
\resizebox{0.48\textwidth}{!}{
\includegraphics{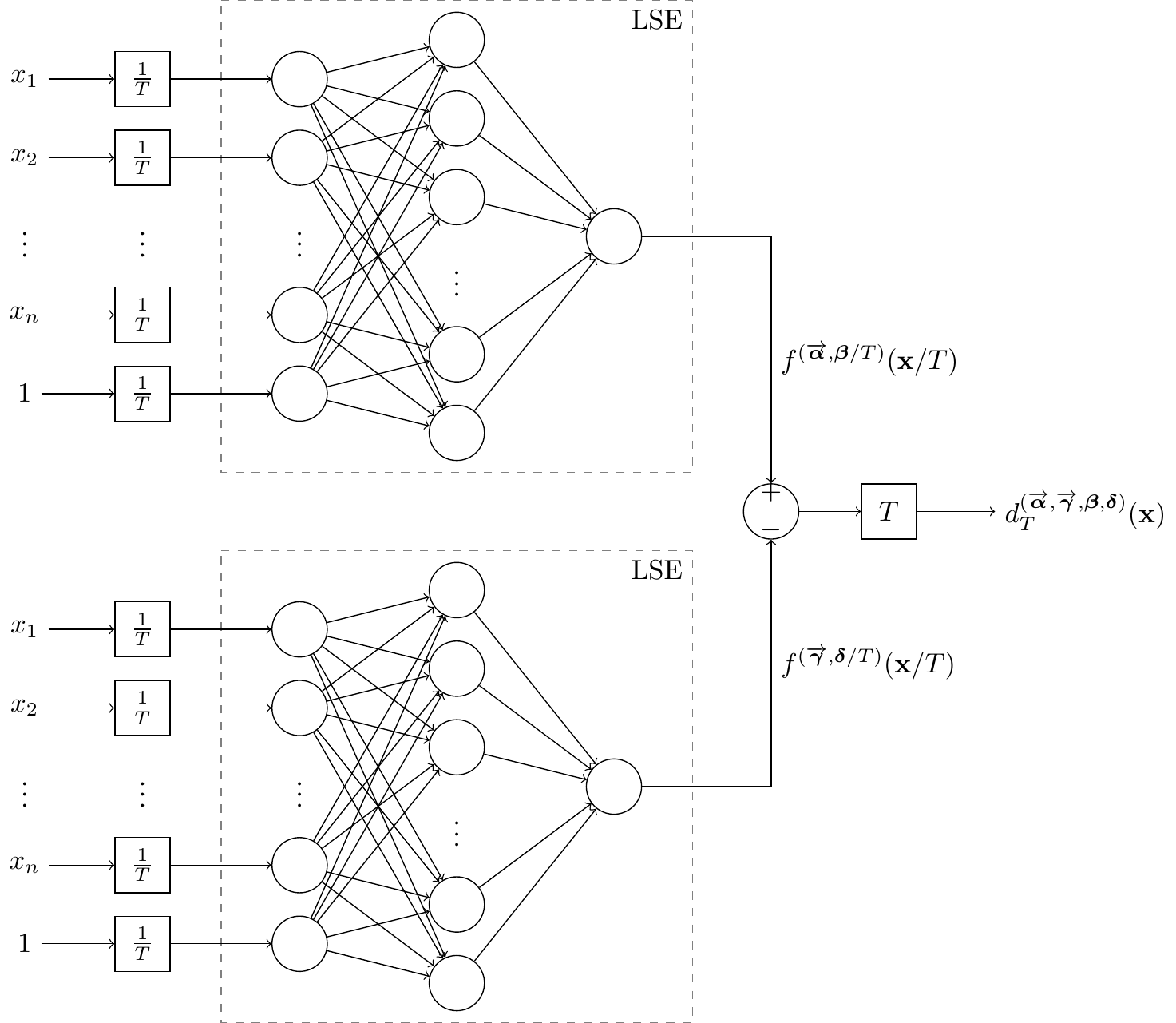}}
\caption{The same  $\dlset$ network 
as in Fig.~\ref{fig:dlsetnet} can be obtained by suitably pre-scaling the input and outputs of the
two component $\lse$ networks. \label{fig:dlsenet}}
\end{figure}

As a matter of fact, given the parameter vectors $\obba$, $\obbg$, $\bbb$,
$\bbd$, it can be easily derived that
\begin{multline*}
d_T^{(\obba,\obbg,\bbb,\bbd)}(\bbx)=f_T^{(\obba,\bbb)}(\bbx)-f_T^{(\obbg,\bbd)}(\bbx)\\
= T \left(f_1^{(\obba,\bbb/T)} (\bbx/T)-f_1^{(\obbg,\bbd/T)} (\bbx/T)\right)\\
=Td_1^{(\obba,\obbg,\bbb/T,\bbd/T)}(\bbx/T),
\end{multline*}
and hence dealing with a $\dlset$ neural network corresponds to deal with a $\dlse$ neural network
whose input and output have been rescaled.
%
%
Each of the two $\lse$ components of the network 
realizes an input-output map of the form
\begin{equation*}
f^{(\obba,\bbb/T)}( \bbx/T) = \log \left(\sum_{k=1}^K \exp ( \inner{\bba^{(k)}}{ \bbx/T} + \beta_k/T )\right).
\label{eq:lsep}
\end{equation*}
The gradients of this function with respect to its parameters $\bba^{(i)}$, $\beta_i$ are, for $i=1,\ldots,K$,
\begin{align*}
\nabla_{\bba^{(i)}} f^{(\obba,\bbb/T)}( \bbx/T) &= \frac {\exp(\inner{\bba^{(i)}}{ \bbx/T} + \beta_i/T )\,{\bbx}}{T\sum_{k=1}^K \exp ( \inner{\bba^{(k)}}{ \bbx/T} + \beta_k /T)}, \\
\nabla_{\beta_i} f^{(\obba,\bbb/T)}( \bbx/T) &=\frac {\exp(\inner{\bba^{(i)}}{ \bbx/T} + \beta_i/T )}{T\sum_{k=1}^K \exp ( \inner{\bba^{(k)}}{ \bbx/T} + \beta_k/T )}.
\end{align*}
Thus, by using the chain rule, we have that
\begin{align*}
\nabla_{\bba^{(i)}} d_T^{(\obba,\obbg,\bbb,\bbd)}(\bbx) &= \frac {\exp(\inner{\bba^{(i)}}{ \bbx/T} + \beta_i/T )\,{\bbx}}{\sum_{k=1}^K \exp ( \inner{\bba^{(k)}}{ \bbx/T} + \beta_k/T )}, \\
\nabla_{\beta_i} d_T^{(\obba,\obbg,\bbb,\bbd)}(\bbx) &=\frac {\exp(\inner{\bba^{(i)}}{ \bbx/T} + \beta_i /T)}{\sum_{k=1}^K \exp ( \inner{\bba^{(k)}}{ \bbx/T} + \beta_k/T )},\\
\nabla_{\bbg^{(i)}} d_T^{(\obba,\obbg,\bbb,\bbd)}(\bbx) &= -\frac {\exp(\inner{\bbg^{(i)}}{ \bbx/T} + \delta_i/T )\,{\bbx}}{\sum_{k=1}^K \exp ( \inner{\bbg^{(k)}}{ \bbx/T} + \delta_k /T)} ,\\
\nabla_{\delta_i} d_T^{(\obba,\obbg,\bbb,\bbd)}(\bbx) &=-\frac {\exp(\inner{\bbg^{(i)}}{ \bbx/T} + \delta_i/T )}{\sum_{k=1}^K \exp ( \inner{\bbg^{(k)}}{ \bbx/T} + \delta_k/T )}.
\end{align*}

Given a dataset $\calD$ as in \eqref{eq:Ddataset}, these gradients can be used to train a $\dlset$ network by using classical algorithms such as 
the Levenberg-Marquardt algorithm \cite{marquardt1963algorithm}, the Fletcher-Powell conjugate gradient \cite{davidon1991variable},
or  the stochastic gradient descent algorithm~\cite{bertsekas1999nonlinear}.
%
In numerical practice, one may fix the parameter $T$ and 
 $K$ and train the network with respect to the
parameters $\obba$, $\bbb$, $\obbg$ and $\bbd$ by using one of the methods mentioned above,
until a satisfactory cross-validated fit is found. A suitable initial value 
for the $T$ parameter may be set, for instance, to the inverse mid output range ${2}/{\vert\max(y_i)-\min(y_i)\vert}$.
Alternatively, $T$ can be considered as a trainable variable as well, and computed by the training algorithm alongside with the other parameters.

\color{black}The following example illustrates the application of the Levenberg-Marquardt algorithm to a simple case.
\begin{exa}\label{ex:nonconvexappr}
Consider the function $\phi:[-2,2]\rightarrow \R$,
\begin{equation}\label{eq:funphi}
\phi(x)=x^2+\sin(2\,\pi\,x).
\end{equation}
Such a function, which is clearly nonconvex, has been used to generate the dataset $\calD= \{(x_i,y_i)\}_{i=1}^{100}$,
where each $x_i$ has been taken uniformly at random in $[-2,2]$ and $y_i=\phi(x_i)$, $i=1,\dots,100$.
The Levenberg-Marquardt algorithm has been used to train a $\dlset$ network fitting such data with $K=10$
and $T={2}/\vert\max(y_i)-\min(y_i)\vert$.
\color{black}Fig.~\ref{fig:dlseEx1} depicts the output of the $\dlset$ network and of its two $\lset$ components
$f_T^{(\obba,\bbb)}( \bbx) $ and $f_T^{(\obbg,\bbd)}( \bbx) $.
\begin{figure}[H]
\centering
\resizebox{0.45\textwidth}{!}{
\includegraphics{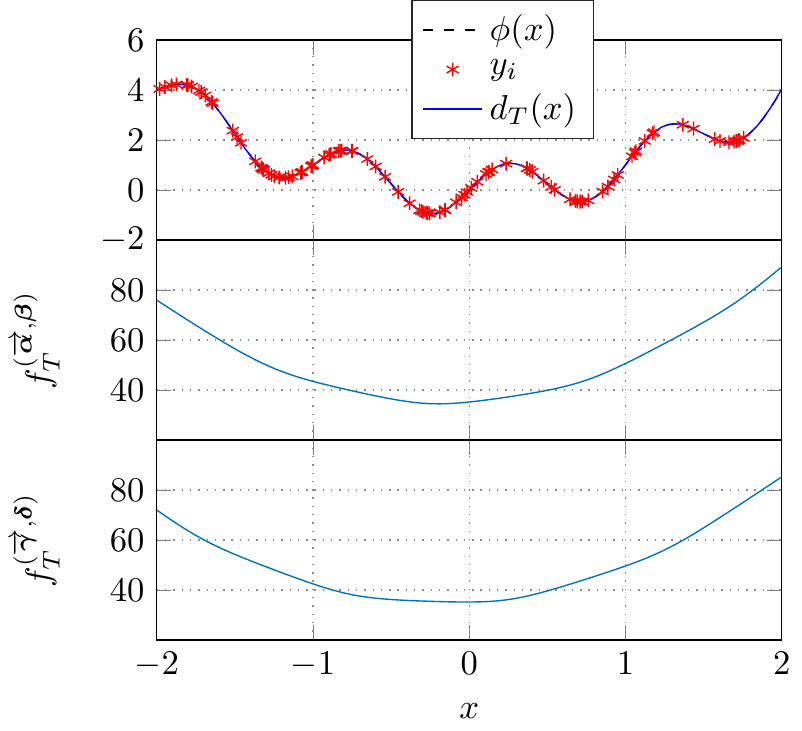}}
\caption{Output of the trained $\dlset$ network.\label{fig:dlseEx1}}
\end{figure}

\noindent 
As shown in Fig.~\ref{fig:dlseEx1}, although the function $\phi$ is nonconvex,
it is well approximated by a $\dlset$ network. 
Indeed, the data represented in Fig.~\ref{fig:dlseEx1} are approximated
by the trained $\dlset$ network with a mean square error of $4.4\cdot 10^{-5}$.
\end{exa}

\section{Non-convex optimization via $\dlset$ networks \label{sec:design}}

In view of the results established in Sections~\ref{sec_DLSE} and~\ref{sec:networks},
$\dlset$ networks can be efficiently used to compute a \emph{difference of convex} (\emph{DC})
approximate decomposition of any continuous function over a compact set. Indeed, by using
the tools described in Section~\ref{sec:networks}, given any continuous function $\phi(\cdot)$ 
defined on a convex compact set $\mK\subset \R^n$ (or, more generally,
a dataset generated through any function $\phi$) 
and $\eps\in\Rpp$, we can determine $g_T,h_T\in\lset$
such that
\begin{equation}\label{eq:appr}
\vert \phi(\bbx)-g_T(\bbx)+h_T(\bbx)\vert \leqslant \eps,\quad \forall\bbx\in\mK.
\end{equation}
Once functions $g_T(\bbx)$ and $h_T(\bbx)$ are determined via training on available data, 
we have a surrogate model 
$d_T = g_T - h_T \simeq \phi$
that we can use for solving approximately  design problems of the form
$\min_{\bbx\in\mK}\phi(\bbx)$,
by substituting the surrogate function $d_T$ in place of the (possibly unknown) function $\phi$.
The resulting surrogate design problem
\begin{equation}\label{eq:minProb}
\min_{\bbx\in\mK}\; d_T^{(\obba,\obbg,\bbb,\bbd)}(\bbx) 
\end{equation}
involves the minimization of the difference of two convex $\lset$ functions.
An approximate solution to this problem can be computed 
by means of a specific and effective algorithm named Difference-of-Convex Algorithm (DCA), which is described in \cite{tao2005dc,le2008dc,dinh2014recent,LeThi2018}.
We next  tailor the DCA  to our specific  context in the  following \Cref{alg:DLSEA}, denoted by DLSEA.

\begin{algorithm}
\caption{Difference of $\lset$ algorithm (DLSEA)\label{alg:DLSEA}}
\begin{algorithmic}[1]
\REQUIRE functions $g_T=f_T^{(\obba,\bbb)}$
	and $h_T=f_T^{(\obbg,\bbd)}$ in $\lset$ and a convex compact set $\mK$.
\ENSURE a candidate optimal solution $\hat{\bbx}^\star$ to the problem 
\begin{equation}\label{eq:minProb2}
\min_{\bbx\in\mK}\; (g_T(\bbx)-h_T(\bbx)).
\end{equation}
\STATE pick initial point $\bbc^{(0)}\in\mK$
\FOR{$\varkappa\in\Zp$}
	\STATE \label{step:vk}let $\bbv^{(\varkappa)}=\frac{\sum_{k=1}^K \exp(\inner{\bbg^{(k)}}{\bbc^{(\varkappa)}/T}+\delta_k/T)\,\bbg^{(k)} }{\sum_{k=1}^K \exp(\inner{\bbg^{(k)}}{\bbc^{(\varkappa)}/T}+\delta_k/T)}$
	\STATE \label{step:xk1}\label{eq:convProb} 
		let
%
$
		\bbc^{(\varkappa+1)}=\argmin_{\bbx\in\mK}\{g_T(\bbx)-\inner{\bbx}{\bbv^{(\varkappa)}}  \}
		$
	\IF{$\frac{\Vert \bbc^{(\varkappa+1)}-\bbc^{(\varkappa)}\Vert}{1+\Vert \bbc^{(\varkappa)}\Vert}$ is smaller than a tolerance}
		\RETURN $\hat{\bbx}^\star=\bbc^{(\varkappa+1)}$
	\ENDIF
\ENDFOR
\end{algorithmic}
\end{algorithm}

The DLSEA returns a candidate optimal solution to problem~\eqref{eq:minProb2}, see \cite{LeThi2018}.
In words, this algorithm starts from a given initial point $\bbc^{(0)}\in\mK$ and iterates until convergence 
Step~\ref{step:vk}, in which the gradient of  $h_T=f_T^{(\obbg,\bbd)}$ is computed as 
  $\bbv^{(\varkappa)}=\nabla h_T(\bbc^{(\varkappa)})$, and 
Step~\ref{step:xk1}, in which a local approximation of  problem \eqref{eq:minProb2} is solved.
We observe that the function $g_T(\bbx)-\inner{\bbx}{\bbv^{(\varkappa)}}$ minimized in Step~\ref{step:xk1} is (except for a constant term that does not affect the minimization) equal
to the difference between  the convex function $g_T(\bbx)$ and the linearization of $h_T$ around the current solution $\bbc^{(\varkappa)}$. Therefore, the problem to be solved in Step~\ref{step:xk1} is a convex minimization problem, which can be solved globally and efficiently via standard tools.

\begin{exa}\label{ex:DLSEA}
The DLSEA has been used to find the global minimum of the function $\phi$ given in~\eqref{eq:funphi}
on the compact convex set $\mK=[-2,2]$. Namely, by exploiting the approximate DC decomposition
of $\phi$ determined in Example~\ref{ex:nonconvexappr} and letting $\bbc^{(0)}=0$, we obtained, 
with $4$ iterations of the DLSEA, the approximate solution $\hat{\bbx}^\star=-0.2381$ to problem~\eqref{eq:minProb2},
that is close to the actual solution $\bbx^\star=-0.2379$.
\end{exa}

\section{Application: diet design for  type 2 diabetes \label{sec:appli}}
Type 2 diabetes mellitus is a chronic disease that affects the way the human body processes glucose.
It is characterized by reduced sensitivity of tissues to insulin, a hormone produced by pancreatic
beta cells that promotes the absorption of glucose from blood into liver, fat and skeletal muscle cells~\cite{ripsin2009management}. 

The main objective of this section is to show how $\dlset$ networks and the DLSEA 
can be used to design a diet based on 5 meals for a patient with type 2 diabetes,
with the aim of minimizing the maximal concentration of glucose in blood,
while guaranteeing that a sufficient amount of glucose is administrated (namely, exactly $185\,\mathrm{g}$).
In order to pursue this objective, the meal model for the glucose-insulin
system given in \cite{dalla2007meal} has been used to simulate the time-behavior of
plasma glucose concentration with 
breakfast at 8 a.m. (containing $x_1\,\mathrm{g}$ of glucose), 
mid-morning snack at  11 a.m. (containing $x_2\,\mathrm{g}$ of glucose), 
lunch at 1 p.m. (containing $x_3\,\mathrm{g}$ of glucose), 
mid-afternoon snack at 5 p.m. (containing $x_4\,\mathrm{g}$ of glucose),
and dinner at 8 p.m. (containing $x_5\,\mathrm{g}$ of glucose).
This model has been used to generate synthetic data: $10^3$ points $\bbx^{(i)}=[\begin{array}{cccc}
x_1^{(i)} & \cdots & x_5^{(i)}\end{array}]^\top$ have been picked uniformly at random from the set
\begin{equation*}
\mK = \left\{\bbx\in\R^5:x_j\geqslant 0,\,j=1,\dots,5,\,\sum_{j=1}^5x_j=185 \right\},
\end{equation*}
and the meal model has been used to determine the corresponding maximum of glucose concentration $y^{(i)}$.
Thus, the training dataset
$\calD_{\mathrm{train}}=\{(\bbx^{(i)},y^{(i)}) \}_{i=1}^{1000}$
has been used to train a $\dlset$ network with temperature parameter $T=12.98\,\cdot 10^{-3}$
(that is ${2}/{\vert\max(y_i)-\min(y_i)\vert}$)
\color{black}and 
with $K=30$ nodes in each of its two $\lset$ components. In order to evaluate the prediction
capabilities of this network, after the training, we generated, by using the same method as above, 
a validation dataset
$\calD_{\mathrm{valid}}=\{(\check{\bbx}^{(i)},\check{y}^{(i)}) \}_{i=1}^{100}$,
and we compared the outputs of the $\dlse_{T}$ model $d_T(\check{\bbx}^{(i)})$ 
with $\check{y}^{(i)} $, $i=1,\dots,100$.
For comparison purposes, a classical $\ffnn$ with symmetric sigmoid activation function for the hidden layer (with $60$ nodes) and linear activation function for the output layer has been trained on the same dataset $\calD_{\mathrm{train}}$ and its prediction performance has been evaluated
over $\calD_{\mathrm{valid}}$. Table~\ref{tab:predErr} summarizes the prediction errors of the two models.

\begin{table}[htb!]
\caption{Prediction errors over $\calD_{\mathrm{valid}}$\label{tab:predErr}}
\centering
{\renewcommand{\arraystretch}{1.2}
\renewcommand{\tabcolsep}{6pt}
\begin{tabular}{cccccc}
\hline
\multirow{2}{*}{Method}  & Mean Sq. & Mean Rel. & Max Abs.& Max Rel. & r$^2$\\
&$[\mathrm{mg/dL}]$ &  $[-]$ & $[\mathrm{mg/dL}]$ & $[-]$ &  $[-]$\\
\hline
$\dlse$  & $0.0624$& $0.0006$ & $1.1043$ & $0.0043$ & $0.9999$\\
$\ffnn$  & $2.3145$& $0.0041$& $7.9480$ & $0.0311$ & $0.9959$\\
\hline
\end{tabular}
}
\end{table}

As shown by Table~\ref{tab:predErr}, the $\dlset$ model has the best performance with respect to all
the error metrics. Furthermore, differently form classical $\ffnn$, it is readily amenable to efficient optimization via
the DLSEA. Indeed, we applied such an algorithm to the $\dlset$ model and we obtained 
the optimal diet that minimizes the maximal concentration of glucose in blood,
while guaranteeing that $185\,\mathrm{g}$ of glucose are administrated.
Fig.~\ref{fig:optimal} depicts the optimal diet and the corresponding time-behavior
of the plasma glucose concentration. The optimal diet  is such that the corresponding maximal
plasma glucose concentration in $24\,\mathrm{h}$ is $253.06\,\mathrm{mg/dL}$.
\begin{figure}[htb!]
\centering
\resizebox{0.5\textwidth}{!}{
\includegraphics{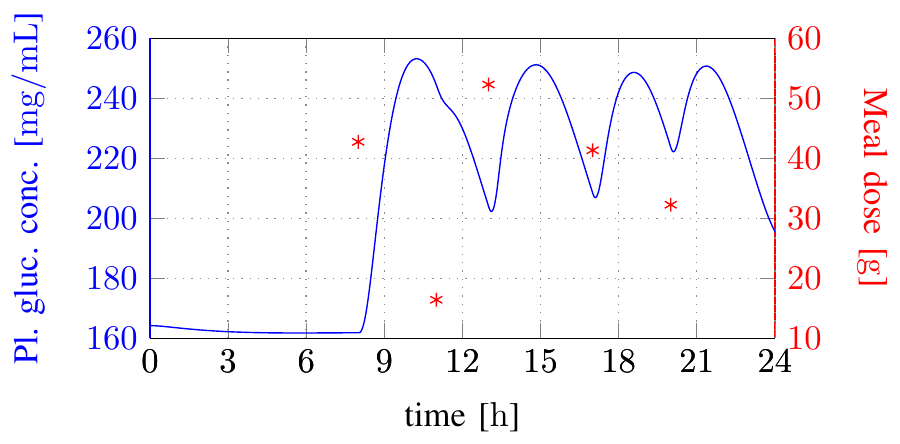}}
\caption{Simulation of the time-behavior of the plasma glucose concentration with the optimal diet determined
via the DLSEA.\label{fig:optimal}}
\end{figure}
%

\section{Conclusions \label{sec:concl}}
In this paper, we showed that a neural network whose output is the difference of the outputs of two
feedforward neural networks with exponential activation function in the hidden layer and logarithmic
activation function in the output node is an universal approximator of continuous functions over
compact convex sets. 
By using a logarithmic transformation, such networks maps to a class of subtraction
free ratios of generalized posynomials, which we showed to be universal approximators of positive
functions over the positive orthant.

The main advantage of $\dlset$ networks with respect to classical $\ffnn$ is that they are readily
amenable to effective optimization-based design. In particular, by adapting the DCA given in \cite{LeThi2018}
to our context, we derived an ad-hoc algorithm for optimizing $\dlset$  models which has proved to be efficient in 
the considered test cases.

\bibliographystyle{ieeetr}
\bibliography{biblio}

\begin{IEEEbiography}
    [{\includegraphics[width=1in,height=3truecm,clip,keepaspectratio]{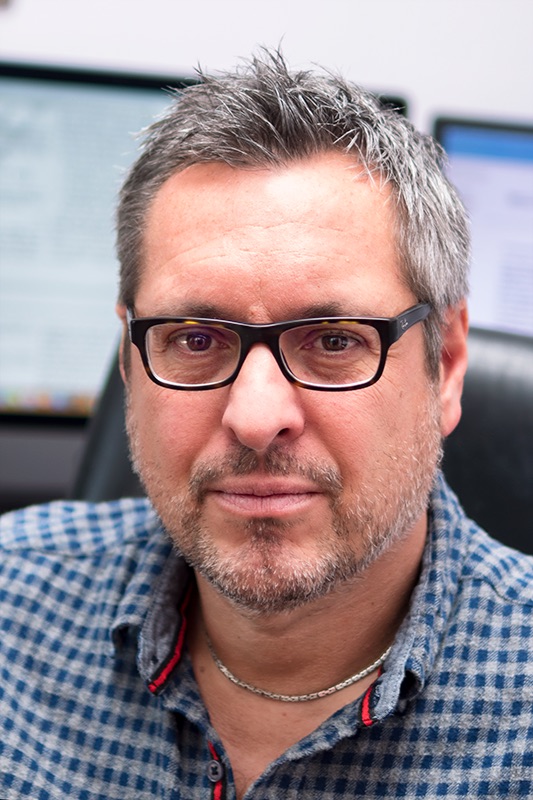}}]{Giuseppe C. Calafiore}
 (S '14, F '18) received the ``Laurea'' degree in Electrical Engineering from Politecnico di Torino in 1993, and the Ph.D. degree in Information and System Theory from Politecnico di Torino, in 1997. He is with the faculty of Dipartimento di Electronica e Telecommunicazioni, Politecnico di Torino, where he currently serves as a full professor and coordinator of the Systems and Data Science lab.
 He is  associated with the Italian National Research Council (CNR). 
Dr. Calafiore held several visiting positions at international institutions: at the Information Systems Laboratory (ISL), Stanford University, California, in 1995; at the Ecole Nationale Sup\'erieure de Techniques Avance\'es (ENSTA), Paris, in 1998; and at the University of California at Berkeley, in 1999, 2003 and 2007. He had an appointment as a Senior Fellow at the Institute of Pure and Applied Mathematics (IPAM), University of California at Los Angeles, in 2010. He had appointments as a Visiting Professor at EECS UC Berkeley in 2017 and at the Haas Business School in 2018 and 2019.
 Dr. Calafiore is the author of more than 180 journal and conference proceedings papers, and of eight books. He is a fellow member of the IEEE since 2018. He received the IEEE Control System Society ``George S. Axelby'' Outstanding Paper Award in 2008.  His research interests are in the fields of convex optimization, randomized algorithms, machine learning, computational finance, and identification and control of uncertain systems.
\end{IEEEbiography}

\begin{IEEEbiography}
    [{\includegraphics[width=1in,height=3truecm,clip,keepaspectratio]{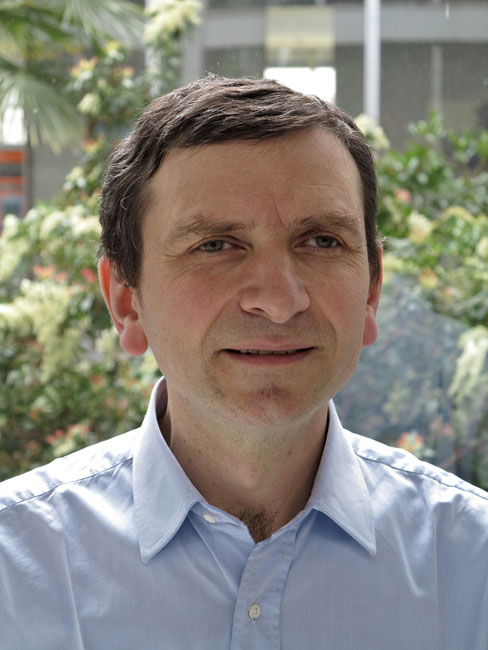}}]{Stephane Gaubert}
     (M '18)
obtained the Engineer degree from \'Ecole
Polytechnique, Palaiseau, in 1988. He got a PhD degree in Mathematics
and Automatic Control from \'Ecole Nationale Sup{\'e}rieure des Mines de
Paris in 1992. He is senior research scientist (Directeur de
Recherche) at INRIA Saclay -- \^Ile-de-France and member of CMAP (Centre
de Math\'ematiques Appliqu\'ees, \'Ecole Polytechnique, CNRS), head of a joint
research team, teaching at \'Ecole Polytechnique.  He coordinates
the Gaspard Monge corporate sponsorship Program for Optimization and Operations Research (PGMO),
of Fondation Math\'ematique Hadamard, Paris-Saclay. 
His interests include tropical geometry, optimization,  game theory, monotone or nonexpansive
dynamical systems, and applications of mathematics to decision making and to the verification of programs
or systems. 
\end{IEEEbiography}

\begin{IEEEbiography}
    [{\includegraphics[width=1in,height=3truecm,clip,keepaspectratio]{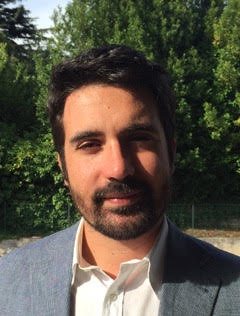}}]{Corrado Possieri}
received his bachelor's and master's degrees in Medical engineering and his Ph.D. degree in Computer Science, Control and Geoinformation from the University of Roma Tor Vergata, Italy, in 2011, 2013, and 2016, respectively. 
From September 2015 to June 2016, he visited the University of California, Santa Barbara (UCSB). 
Currently, he is Assistant Professor at the Politecnico di Torino.
He is a member of the IFAC TC on Control Design.
His research interests include stability and control of hybrid systems, the application of computational algebraic geometry techniques to control problems, stochastic systems, and optimization.
\end{IEEEbiography}

\end{document}